\documentclass[11pt,letterpaper]{article}
\pdfoutput=1
\usepackage{emnlp2016}
\usepackage{times}
\usepackage{latexsym}
\usepackage{hyperref}
\usepackage{amsmath,graphicx,xcolor,defs,algorithmic, boxedminipage, float}
\usepackage{afterpage}
\usepackage{hyphenat}
\usepackage{soul}
\usepackage{url}
\usepackage{amsthm}

\emnlpfinalcopy

\usepackage[utf8]{inputenc} % allow utf-8 input
\usepackage[T1]{fontenc}    % use 8-bit T1 fonts
\usepackage{hyperref}       % hyperlinks
\usepackage{url}            % simple URL typesetting
\usepackage{booktabs}       % professional-quality tables
\usepackage{amsfonts}       % blackboard math symbols
\usepackage{nicefrac}       % compact symbols for 1/2, etc.
\usepackage{microtype}      % microtypography

\def\emnlppaperid{***}

\date{}
\title{Length bias in Encoder Decoder Models and a Case for Global Conditioning}

\author{
  Pavel Sountsov\\
  Google \\
  \texttt{siege@google.com} \\
\And
  Sunita Sarawagi~\thanks{\hspace{0.5em} Work done while visiting Google Research on a leave from IIT Bombay.}\\
  IIT Bombay \\
  \texttt{sunita@iitb.ac.in} \\
}

\begin{document}

\maketitle

\begin{abstract}
Encoder-decoder networks are popular for modeling
sequences  probabilistically in many applications.  These models use the power of the
Long Short-Term Memory (LSTM) architecture to capture the {\em full
  dependence} among variables, unlike earlier models like CRFs that
typically assumed conditional independence among non-adjacent
variables.  However in practice encoder-decoder models exhibit a bias
towards short sequences that surprisingly gets worse with increasing
beam size.
% search to find the optimal sequence. Surprisingly, sometimes there is even a decline in accuracy with increasing the beam size.

In this paper we show that such phenomenon is due to a discrepancy
between the full sequence margin and the per-element margin enforced by
the locally conditioned training objective of a encoder-decoder model.
The discrepancy more adversely impacts long sequences, explaining the
bias towards predicting short sequences.

For the case where the predicted sequences come from a closed set, we show that
a globally conditioned model alleviates the above problems of
encoder-decoder models.  
From a practical point of view, our proposed model also eliminates the
need for a beam-search during inference, which reduces to an efficient
dot-product based search in a vector-space.
\end{abstract}

\section{Introduction}
In this paper we investigate the use of neural networks for
modeling the conditional distribution $\Pr(\vy | \vx)$ over sequences
$\vy$ of discrete tokens in response to a complex input $\vx$, which can be
another sequence or an image.  Such models have applications in
machine translation~\cite{BahdanauCB14,SutskeverVL14}, image captioning~\cite{VinyalsTBE15}, response generation in emails~\cite{smartreply}, and conversations~\cite{smartchat,VinyalsL15,LiGBGD15}.

The most popular neural network for probabilistic modeling of
sequences in the above applications is the encoder-decoder (ED)
network~\cite{SutskeverVL14}.  A ED network first encodes an input
$\vx$ into a vector which is then used to initialize a recurrent
neural network (RNN) for decoding the output $\vy$.  The decoder RNN
factorizes $\Pr(\vy|\vx)$ using the chain rule as
$\prod_j \Pr(y_j|y_1,\ldots,y_{j-1},\vx)$ where $y_1,\ldots,y_n$
denote the tokens in $\vy$.
This factorization does not entail any conditional independence
assumption among the $\{y_j\}$ variables.  This is unlike earlier
sequence models like CRFs~\cite{Lafferty2001} and
MeMMs~\cite{Mccallum00} that typically assume that a token is
independent of all other tokens given its adjacent tokens.  Modern-day
RNNs like LSTMs promise to capture non-adjacent and long-term
dependencies by summarizing the set of previous tokens in a
continuous, high-dimensional state vector.  Within the limits of
parameter capacity allocated to the model, the ED, by virtue of exactly
factorizing the token sequence, is consistent.

However, when we created and deployed an ED model for a chat
suggestion task we observed several counter-intuitive patterns in its
predicted outputs.  Even after training the model over billions of
examples, the predictions were systematically biased towards short
sequences.  Such bias has also been seen in
translation~\cite{ChoMBB14}.  Another curious phenomenon was that the
accuracy of the predictions sometimes dropped with increasing
beam-size, more than could be explained by statistical variations of a
well-calibrated model~\cite{marc2016sequence}.

In this paper we expose a margin discrepancy in the training loss
of encoder-decoder models to explain the above problems in its
predictions.  We show that the training loss of ED network often
under-estimates the margin of separating a correct sequence from an
incorrect shorter sequence.  The discrepancy gets more severe as the
length of the correct sequence increases.  That is, even after the
training loss converges to a small value, full inference on the
training data can incur errors causing the model to be under-fitted
for long sequences in spite of low training cost.  We call this the
length bias problem.

We propose an alternative model that avoids the margin discrepancy by
globally conditioning the $P(\vy|\vx)$ distribution.  Our model is
applicable in the many practical tasks where the space of allowed
outputs is closed.  For example, the responses generated by the smart
reply feature of Inbox is restricted to lie within a hand-screened
whitelist of responses $\cW \subset \cY$~\cite{smartreply}, and the
same holds for a recent conversation assistant feature of Google's
Allo~\cite{smartchat}.  Our model uses a second RNN encoder to
represent the output as another fixed length vector.  We show that our
proposed encoder-encoder model produces better calibrated whole
sequence probabilities and alleviates the length-bias problem of ED
models on two conversation tasks.
A second advantage of our model is that inference is significantly
faster than ED models and is guaranteed to find the globally optimal
solution.  In contrast, inference in ED models requires an expensive
beam-search which is both slow and is not guaranteed to find the optimal
sequence.

\section{Length Bias in Encoder-Decoder Models}
\newtheorem{theorem}{Theorem}[section]
\newtheorem{corollary}{Corollary}[theorem]
\newtheorem{lemma}[theorem]{Lemma}
\newtheorem{claim}[theorem]{Claim}

In this section we analyze the widely used encoder-decoder neural
network for modeling $\Pr(\vy | \vx)$ 
over the space of discrete output sequences. 
We use $y_1,\ldots,y_n$ to denote the tokens in a sequence $\vy$.
Each $y_i$ is a discrete symbol from a finite dictionary $V$ of size
$m$.  Typically, $m$ is large.
The length $n$ of a sequence is allowed to vary from sequence to
sequence even for the same input $\vx$.  A special token EOS $\in V$
is used to mark the end of a sequence.  We use $\cY$ to denote the
space of such valid sequences and $\theta$ to denote the parameters of
the model.

\subsection{The encoder-decoder network}
The Encoder-Decoder (ED) network represents $\Pr(\vy | \vx, \theta)$
by applying chain rule to exactly factorize it as
$\prod_{t=1}^n\Pr(y_t | y_1,\ldots,y_{t-1},\vx, \theta)$.  First, an encoder
with parameters $\theta_x \subset \theta$ is used to transform $\vx$
into a $d$-dimensional real-vector $\vv_x$.  The network used for the
encoder depends on the form of $\vx$ --- for example, when $\vx$ is
also a sequence, the encoder could be a RNN.
The decoder then computes each $\Pr(y_t | y_1,\ldots,y_{t-1},\vv_x,
\theta)$ as
\begin{equation}
\Pr(y_t | y_1,\ldots,y_{t-1},\vv_x,
\theta) = P(y_t | \vs_t, \theta),
\end{equation}
where $\vs_t$ is a state vector implemented using a recurrent neural network as
\begin{equation}
\label{rnn-eq}
  \vs_t = 
      \begin{cases}
      \vv_x & \text{if}~ t  = 0, \\
      \text{RNN}(\vs_{t-1}, \theta_{E,y_{t-1}}, \theta_R) & \text{otherwise}.
    \end{cases}
\end{equation}
where RNN() is typically a stack of LSTM cells that captures long-term
 dependencies, $\theta_{E,y} \subset \theta$ are parameters denoting
 the embedding for token $y$, and $\theta_R \subset \theta$ are the
 parameters of the RNN.  The function $\Pr(y | \vs, \theta_y)$ that
 outputs the distribution over the $m$ tokens is a softmax:
\begin{equation}
  \label{softmax-eq}
  \Pr(y | \vs, \theta) = \frac{e^{\vs\theta_{S,y}}}{e^{\vs\theta_{S,1}}+\ldots+e^{\vs\theta_{S,m}}},
\end{equation}
where $\theta_{S,y} \subset \theta$ denotes the parameters for token
$y$ in the final softmax.

\subsection{The Origin of Length Bias}
\label{sec:length-bias}
The ED network builds a single probability distribution over sequences
of arbitrary length.  For an input $\vx$, the network needs to choose
the highest probability $\vy$ among valid candidate sequences of
widely different lengths.  Unlike in applications like entity-tagging
and parsing where the length of the output is determined based on the
input, in applications like response generation valid outputs can
be of widely varying length.  Therefore, $\Pr(\vy|\vx, \theta)$ should
be well-calibrated over all sequence lengths.  Indeed under infinite
data and model capacity the ED model is consistent and will represent
all sequence lengths faithfully.  In practice when training data is
finite, we show that the ED model is biased against long sequences.
Other researchers~\cite{ChoMBB14} have reported this bias but we are
not aware of any analysis like ours explaining the reasons of this
bias.

\begin{claim}
  The training loss of the ED model under-estimates the margin of
  separating long sequences from short ones.
\end{claim}

\begin{proof}
Let $\vx$ be an input for which a correct output $\vy^+$ is of
length $\ell$ and an incorrect output $\vy^-$ is of length 1.
Ideally, the training loss should put a positive margin 
between $\vy^+$ and $\vy^-$ which is $\log\Pr(\vy^+|\vx)
- \log \Pr(\vy^-|\vx)$.  Let us investigate if the maximum likelihood
training objective of the ED model achieves that. We can write this objective as:
\begin{equation}
\label{eq:edobj}
\max_\theta \log\Pr(y_1^+|\vx,\theta) + \sum_{j=2}^\ell \log\Pr(y_j^+|y_{1\ldots j-1}^+,\vx,\theta).
\end{equation}
Only the first term in the above objective is involved in enforcing a
margin between $\vy^+$ and $\vy^-$ because $\log\Pr(y_1^+|\vx)$ is
maximized when $\log\Pr(y_1^-|\vx)$ is correspondingly minimized.  Let
$m_L(\theta) = \log\Pr(y_1^+|\vx,\theta) - \log\Pr(y_1^-|\vx,\theta)$,
the local margin from the first position and $m_R(\theta)
=\sum_{j=2}^\ell \log\Pr(y_j^+|\vy_{1\ldots j-1}^+,\vx,\theta)$.
It is easy to see that our desired margin between $\vy^+$ and $\vy^-$
is $\log\Pr(\vy^+|\vx) - \log\Pr(\vy^-|\vx) = m_L + m_R$. Let $m_g =
m_L + m_R$.  Assuming two possible labels for the first position
($m=2$)~\footnote{For $m>2$, the objective will be upper bounded by $\min_\theta\log(1 + (m-1)e^{-m_L(\theta)}) - m_R(\theta)$. The argument that follows remains largely unchanged}, the training objective in Equation~\ref{eq:edobj} can now be
rewritten in terms of the margins as:
\[\min_\theta \log(1+e^{-m_L(\theta)}) - m_R(\theta)\]
We next argue that this objective is not aligned with our ideal goal of
making the global margin $m_L+m_R$ positive.

First, note that $m_R$ is a log probability which under finite
parameters will be non-zero.  Second, even though $m_L$ can take any
arbitrary finite value, the training objective drops rapidly when
$m_L$ is positive.
When training objective is regularized and training data is finite,
the model parameters $\theta$ cannot take very large values and the
trainer will converge at a small positive value of $m_L$.
Finally, we
show that the value of $m_R$ decreases with increasing sequence
length.  For each position $j$ in the sequence, we add to $m_R$
log-probability of $y_j^+$.  The maximum value of
$\log\Pr(y_j^+|\vy_{1\ldots j-1}^+,\vx,\theta)$ is $\log(1-\epsilon)$
where $\epsilon$ is non-zero and decreasing with the magnitude of the
parameters $\theta$.  In general, $\log\Pr(y_j^+|\vy_{1\ldots
j-1}^+,\vx,\theta)$ can be a much smaller negative value when the
input $\vx$ has multiple correct responses as is common in
conversation tasks.  For example, an input like $\vx=$`How are you?',
has many possible correct outputs: $\vy \in $\{`I am good', `I am
great', `I am fine, how about you?', etc\}.  Let $f_j$ denote the
relative frequency of output $y_j^+$ among all correct responses with
prefix $\vy_{1\ldots j-1}^+$.   The value of $m_R$ will be upper bounded as
\[m_R \le \sum_{j=2}^{\ell}\log\min(1-\epsilon, f_j)\]
This term is negative always and increases in magnitude as sequence
length increases and the set of positive outpus have high entropy.
In this situation, when combined with regularization, our desired margin $m_g$
may not remain positive even though $m_L$ is positive. In summary, the core
issue here is that since the ED loss is optimized and regularized on the
local problem it does not control for the global, task relevant margin.
\end{proof}

%% We illustrate a worst case value of $m_R$.
%% Let the set of correct outputs for an $\vx$ be of length $\ell$ with a
%% `1' at the first position and `0' or `1' with equal frequency in all
%% other positions.  Therefore, $m_R$ will be estimated to be close to
%% $-\ell\log 2$.  This implies that our desired margin $m_g = m_L + m_R$
%% could be reduced by up to $\ell\log 2$ and may not remain positive
%% even though $m_L$ is positive.

This mismatch between the local margin optimized during training and
the global margin explains the length bias observed by us and others
~\cite{ChoMBB14}.  During inference a shorter sequence for which $m_R$
is smaller wins over larger sequences.

This mismatch also explains
why increasing beam size leads to a drop in accuracy
sometimes~\cite{marc2016sequence}\footnote{Figure 6 in the paper shows
a drop in BLEU score by 0.5 as the beam size is increased from 3
to 10.}.
When beam size is large, we are more likely to dig out short sequences
that have otherwise been separated by the local margin.  We show
empirically in Section~\ref{sec:results} that for long sequences
larger beam size hurts accuracy whereas for small sequences the effect
is the opposite.

\subsection{Proposed fixes to the ED models}
Many ad hoc approaches have been used to alleviate length bias
directly or indirectly. Some resort to normalizing the probability by
the full sequence length \cite{ChoMBB14,Graves13} whereas
\cite{AbadieBMCB14} proposes segmenting longer sentences into shorter phrases.
\cite{ChoMBB14} conjectures that the length bias
of ED models could be because of limited representation power of the
encoder network. Later more powerful encoders based on attention
achieved greater accuracy~\cite{BahdanauCB14} on long sequences.
Attention can be viewed as a mechanism of improving the capacity of
the local models, thereby making the local margin $m_L$ more
definitive.  But attention is not effective for all tasks --- for
example, \cite{VinyalsL15} report that attention was not useful for
conversation.  

Recently~\cite{Bengio2015,marc2016sequence} propose another
modification to the ED training objective where the true token
$y_{j-1}$ in the training term $\log\Pr(y_j|y_1,\ldots,y_{j-1})$ is
replaced by a sample or top-k modes from the posterior at position
$j-1$ via a careful schedule.  Incidently, this fix also helps to
indirectly alleviate the length bias problem.  The sampling causes
incorrect tokens to be used as previous history for producing a
correct token.  If earlier the incorrect token was followed by a
low-entropy EOS token, now that state should also admit the correct
token causing a decrease in the probability of EOS, and therefore the
short sequence.

In the next section we propose our more direct fix to the margin
discrepancy problem.

\section{Globally Conditioned Encoder-Encoder Models}
We represent $\Pr(\vy|\vx,\theta)$ as a globally conditioned model
$\frac{e^{s(\vy|\vx,\theta)}}{Z(\vx,\theta)}$ where
$s(\vy|\vx,\theta)$ denotes a score for output $\vy$ and
$Z(\vx,\theta)$ denotes the shared normalizer.  We show in
Section~\ref{sec:margin} why such global conditioning solves the
margin discrepancy problem of the ED model.  The intractable
partition function in global conditioning introduces several new
challenges during training and inference.  In this section we discuss
how we designed our network to address them.
\begin{figure*}[ht]
  \includegraphics[width=\textwidth]{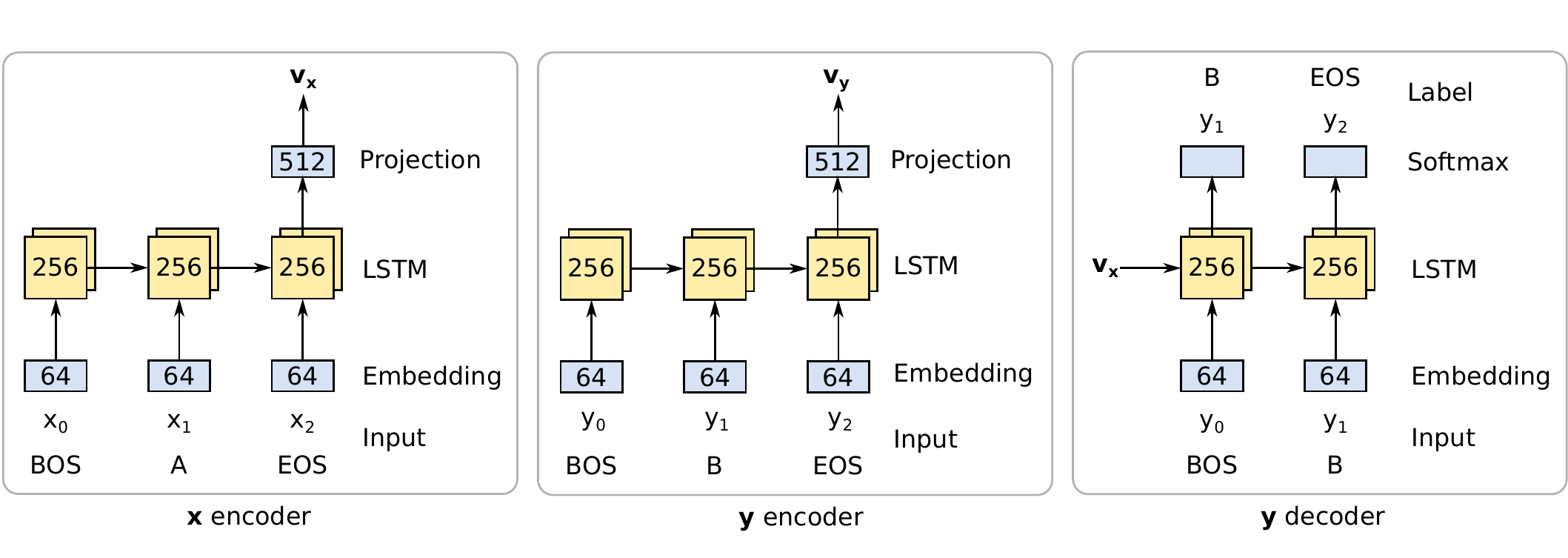}
    \caption{Neural network architectures used in our experiments. The context
    encoder network is used for both encoder-encoder and encoder-decoder models
    to encode the context sequence (`A') into a $\vv_x$. For the
    encoder-encoder model, label sequence (`B') are encoded into $\vv_y$ by the
    label encoder network. For the encoder-decoder network, the label sequence
    is decomposed using the chain rule by the decoder network.}
  \label{fig:networkdiagram}
\end{figure*}

Our model assumes that during inference the output has to be selected
from a given whitelist of responses $\cW \subset \cY$. In spite of
this restriction, the problem does not reduce to multi-class
classification because of two important reasons.  First, during
training we wish to tap all available input-output pairs including the
significantly more abundant outputs that do not come from the
whitelist.  Second, the whitelist could be very large and treating
each output sequence as an atomic class can limit generalization
achievable by modeling at the level of tokens in the sequence.

\subsection{Modeling $s(\vy|\vx,\theta)$}
We use a second encoder to convert $\vy$ into a vector $\vv_y$ of the
same size as the vector $\vv_x$ obtained by encoding $\vx$ as in a ED
network. The
parameters used to encode $\vv_x$ and $\vv_y$ are disjoint. As we are only interested in a fixed dimensional output, unlike in
ED networks, we have complete freedom in choosing the type of network to use
for this second encoder. For our experiments, we have chosen to use an RNN with
LSTM cells. Experimenting with other network architectures, such as
bidirectional RNNs remains an interesting avenue for future work. The score
$s(\vy|\vx,\theta)$ is the dot-product between $\vv_y$ and $\vv_x$.  Thus our model is
\begin{equation}
    \Pr(\vy|\vx) = \frac{e^{\vv_x^T\vv_y}}{{\sum_{\vy' \in \cY} e^{\vv_x^T\vv_{y'}}}}.
\end{equation}

\subsection{Training and Inference}
During training we use maximum likelihood to estimate $\theta$ given
a large set of valid input-output pairs
$\{(\vx^1,\vy^1),\ldots,(\vx^N, \vy^N)\}$ where each $\vy^i$ belongs
to $\cY$ which in general is much larger than $\cW$.  Our main
challenge during training is that $\cY$ is intractably
large for computing $Z$.  We decompose $Z$ as
\begin{equation}
    Z = e^{s(\vy|\vx,\theta)} + \sum_{\vy' \in \cY \setminus \vy} e^{s(\vy'|\vx,\theta)},
\end{equation}
and then resort to estimating the last term using importance sampling.
Constructing a high quality proposal distribution over $\cY \setminus \vy$ is
difficult in its own right, so in practice, we make the following
approximations. We extract the most common $T$ sequences across a data set
into a pool of negative examples. We estimate the empirical prior
probability of the sequences in that pool, $Q(\vy)$, and then draw $k$
samples from this distribution. We take care to remove the true sequence from
this distribution so as to remove the need to estimate its prior probability.

During inference, given an input $\vx$ we need to find $\argmax_{\vy
  \in \cW} s(\vy|\vx,\theta)$.  This task can be performed efficiently in
our network because the vectors $\vv_y$ for the sequences $\vy$ in the
whitelist $\cW$ can be pre-computed.  Given an input $\vx$, we compute
$\vv_x$ and take dot-product with the pre-computed vectors to find the
highest scoring response.  This gives us the optimal response.  When
$\cW$ is very large, we can obtain an approximate solution by indexing
the vectors $\vv_y$ of $\cW$ using recent methods specifically
designed for dot-product based retrieval ~\cite{Guo16}.

\subsection{Margin}
\label{sec:margin}

It is well-known that the maximum likelihood training objective of a
globally normalized model is margin maximizing~\cite{RossetZH03}. We
illustrate this property using our set up from Claim~2.1 where a
correct output $\vy^+$ is of length $\ell$ and an incorrect output
$\vy^-$ is of length 1 with two possible labels for each position
($m=2$).

The globally conditioned model learns a parameter per possible sequence and
assigns the probability to each sequence using a softmax over those parameters.
Additionally, we place a Gaussian prior on the parameters with a precision $c$.
The loss for a positive example becomes:

\begin{equation*}
    \cL_G(\vy^+) = -\log \frac{e^{-\theta_{\vy^+}}}{\sum_{\vy'}e^{-\theta_{\vy'}}} + \frac{c}{2} \sum_{\vy'}\theta_{\vy'}^2,
\end{equation*}
where the sums are taken over all possible sequences.

We also train an ED model on this task. It also learns a parameter for
every possible sequence, but assigns probability to each sequence
using the chain rule.  We also place the same Gaussian prior as above
on the parameters. Let $\vy_j$ denote the first $j$ tokens $\{y_1,
\ldots, y_j\}$ of sequence $\vy$. The loss for a positive example for
this model is then:
\begin{equation*}
    \cL_L(\vy^+) = -\sum_{j=1}^\ell \left( \log \frac{e^{-\theta_{\vy^+_j}}}{\sum_{\vy'_j}e^{-\theta_{\vy'_j}}} + \frac{c}{2} \sum_{\vy'_j}\theta_{\vy'_j}^2 \right),
\end{equation*}
where the inner sums are taken over all sequences of length $j$.

\begin{figure}[ht]
\begin{center}
  \includegraphics[width=3in]{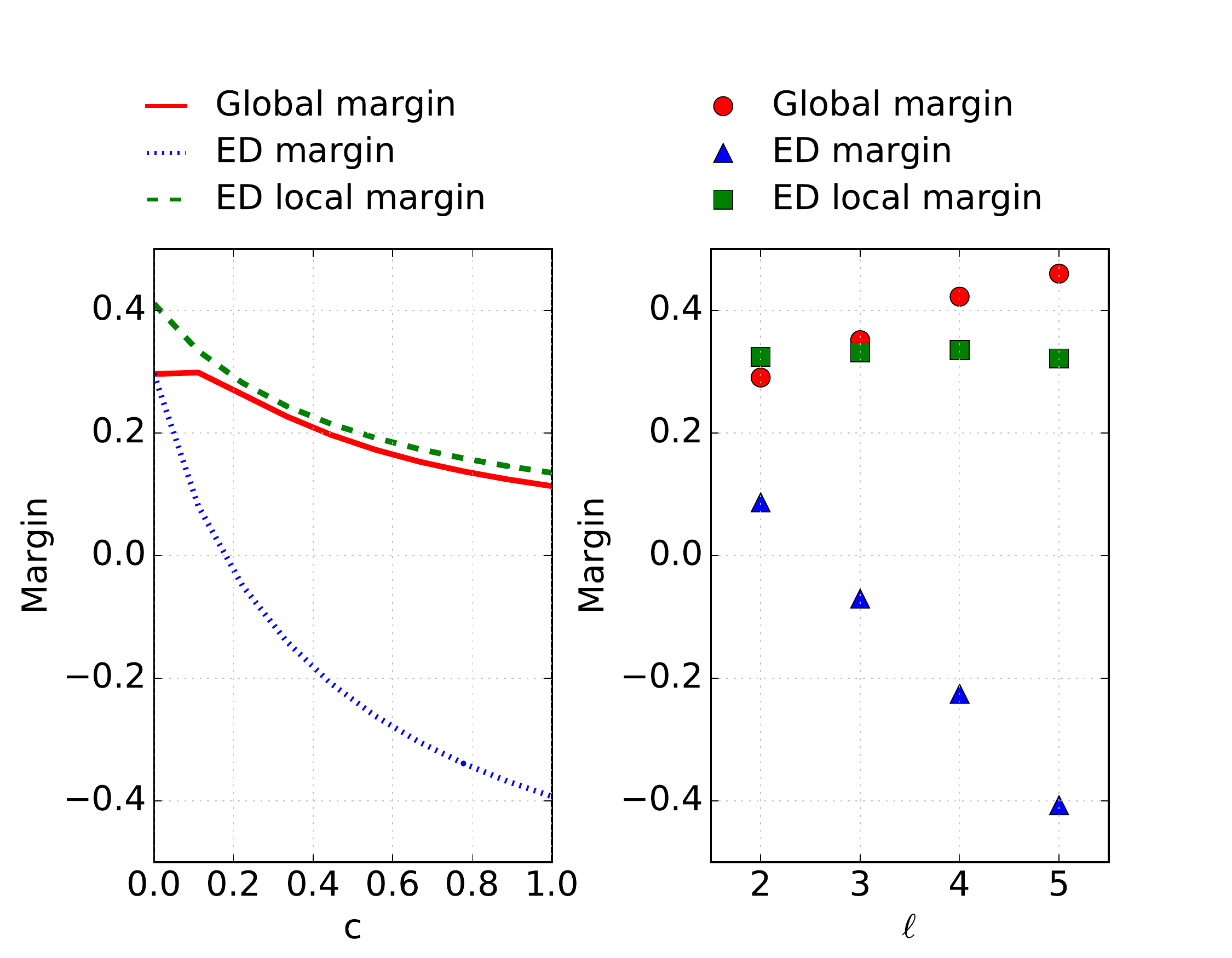}
  \caption{Comparing final margins of ED model with a globally conditioned
    model on example dataset of Section~\ref{sec:margin} as a function of
    regularization constant $c$ and message length $\ell$.}
  \label{fig:margin_toy}
\end{center}
\end{figure}

We train both models on synthetic sequences generated using the following rule.
The first token is chosen to be `1' probability 0.6. If `1' is chosen,
it means that this is a positive example and the remaining $\ell - 1$ tokens
are chosen to be `1' with probability $0.9^\frac{1}{\ell - 1}$. If
a `0' is chosen as the first token, then that is a negative example, and the
sequence generation does not go further. This means that there are $2^{\ell -
1}$ unique positive sequences of length $\ell$ and one negative sequence of length
1. The remaining possible sequences do not occur in the training or testing data.
By construction the unbiased margin between the most probable correct example
and the incorrect example is length independent and positive. We sample 10000
such sequences and train both models using Adagrad \cite{Duchi2011} for 1000
epochs with a learning rate of 0.1, effectively to convergence.

Figure~\ref{fig:margin_toy} shows the margin for both models (between the most
likely correct sequence and the incorrect sequence) and the local margin for
the ED model at the end of training. On the left panel, we used sequences with
$\ell = 2$ and varied the regularization constant $c$. When $c$ is zero, both
models learn the same global margin, but as it is increased the margin for the
ED model decreases and becomes negative at $c > 0.2$, despite the local margin
remaining positive and high. On the right panel we used $c = 0.1$ and varied
$\ell$. The ED model becomes unable to separate the sequences with length
above 2 with this regularization constant setting.

\section{Experiments}
\subsection{Datasets and Tasks}
We contrast the quality of the ED and encoder-encoder models on
two conversational datasets: Open Subtitles and Reddit Comments.

\subsubsection{Open Subtitles Dataset}

The Open Subtitles dataset consists of transcriptions of spoken dialog in
movies and television shows \cite{Lison2016}. We restrict our modeling only to
the English subtitles, of which results in $319$ million utternaces. Each
utterance is tokenized into word and punctuation tokens, with the start and end
marked by the BOS and EOS tokens. We randomly split out $90\%$ of the
utterances into the training set, placing the rest into the validation set. As
the speaker information is not present in this data set, we treat each
utterance as a label sequence, with the preceding utterances as context.

\subsubsection{Reddit Comments Dataset}

The Reddit Comments dataset is constructed from publicly available user
comments on submissions on the Reddit website. Each submission is associated
with a list of directed comment trees. In total, there are $41$ million
submissions and $501$ million comments. We tokenize the individual comments in
the same way as we have done with the utternaces in the Open Subtitles dataset.
We randomly split $90\%$ of the submissions and the associated comments into
the training set, and the rest into the validation set. We use each comment
(except the ones with no parent comments) as a label sequence, with the context
sequence composed of its ancestor comments.

\subsubsection{Whitelist and Vocabulary}

From each dataset, we derived a dictionary of $20$ thousand
most commonly used tokens. Additionally, each dictionary contained the unknown token (UNK), BOS and EOS tokens. Tokens in the datasets which were not present in their associated vocabularies were replaced by the UNK token.

From each data set, we extracted $10$ million most common label
sequences that also contained at most $100$ tokens. This set of
sequences was used as the negative sample pool for the encoder-encoder
models.  For evaluation we created a whitelist $\cW$ out of the 
$100$ thousand most common sequences. We removed any sequence from this set that contained any UNK tokens to simplify inference.

\subsubsection{Sequence Prediction Task}

To evaluate the quality of these models, we task them to predict the true label
sequence given its context. Due to the computational expense, we sub-sample the
validation data sets to around $1$ million context-label pairs. We additionally
restrict the context-label pairs such that the label sequence is present in the
evaluation set of common messages. We use recall@K as a measure of accuracy of
the model predictions. It is defined as the fraction of test pairs where the
correct label is within K most probable predictions according to the model. For
encoder-encoder models we use an exhaustive search over the evaluation set of
common messages. For ED models we use a beam search with width ranging from $1$
to $15$ over a token prefix trie constructed from the sequences in $\cW$.

\subsection{Model Structure and Training Procedure}

The context encoder, label encoder and decoder are implemented using LSTM
recurrent networks \cite{Hochreiter1997} with peephole connections
\cite{Sak2014}. The context and label token sequences were mapped to embedding
vectors using a lookup table that is trained jointly with the rest of the model
parameters. The recurrent nets were unrolled in time up to $100$ time-steps,
with label sequences of greater length discarded and context sequences of greater
length truncated.

The decoder in the ED model is trained by using the true label sequence prefix
as input, and a shifted label sequence as output \cite{SutskeverVL14}. The
partition function in the softmax over tokens is estimated using importance
sampling with a unigram distribution over tokens as the proposal distribution
\cite{Jean2014}. We sample $512$ negative examples from $Q(\vy)$ to estimate
the partition function for the encoder-encoder model. See
Figure~\ref{fig:networkdiagram} for connectivity and network size details.

All models were trained using Adagrad \cite{Duchi2011} with an initial base
learning rate of $0.1$ which we exponentially decayed with a decade of $15$
million steps. For stability, we clip the L2 norm of the gradients to a
maximum magnitude of $1$ as described in \cite{Pascanu2012}. All models are
trained for $30$ million steps with a mini-batch size of 64. The models are trained in a
distributed manner on CPUs and NVidia GPUs using TensorFlow \cite{Abadi2015}.

\subsection{Results}
\label{sec:results}
We first demonstrate the discrepancy between the local and global
margin in the ED models as discussed in Section~\ref{sec:margin}.  We
used a beam size of 15 to get the top prediction from our trained ED
models on the test data and focussed on the subset for which the top
prediction was incorrect.  We measured local and global margin between
the top predicted sequence ($\vy^-$) and the correct test sequence
($\vy^+$) as follows: Global margin is the difference in their full
sequence log probability. Local margin is the difference in the local
token probability of the smallest position $j$ where $y_j^- \ne
y_j^+$, that is local margin is $\Pr(y_j^+|\vy_{1\ldots
  j-1}^+,\vx,\theta) - \Pr(y_j^-|\vy_{1\ldots j-1}^+,\vx,\theta)$.
Note the training loss of ED models directly compares only the local
margin.

\paragraph*{Global margin is much smaller than local margin}
In Figure~\ref{fig:margin-scatter} we show the local and global margin
as a 2D histogram with color luminosity denoting frequency.  We observe
that the global margin values are much smaller than the local
margins. The prominent spine is for $(\vy^+, \vy^-)$ pairs differing
only in a single position making the local and global margins
equal. Most of the mass is below the spine.  For a significant
fraction of cases (27\% for Reddit, and 21\% for Subtitles), the local
margin is positive while the global margin is negative.  That is, the
ED loss for these sequences is small even though the log-probability
of the correct sequence is much smaller than the log-probability of
the predicted wrong sequence.

\paragraph*{Beam search is not the bottleneck}
An interesting side observation from the plots in
Figure~\ref{fig:margin-scatter} is that more than 98\% of the wrong
predictions have a negative margin, that is, the score of the correct
sequence is indeed lower than the score of the wrong prediction.
Improving the beam-width beyond 15 is not likely to improve these
models since only in 1.9\% and 1.7\% of the cases is the correct score
higher than the score of the wrong prediction.

\begin{figure}[ht]
  \includegraphics[width=0.23\textwidth]{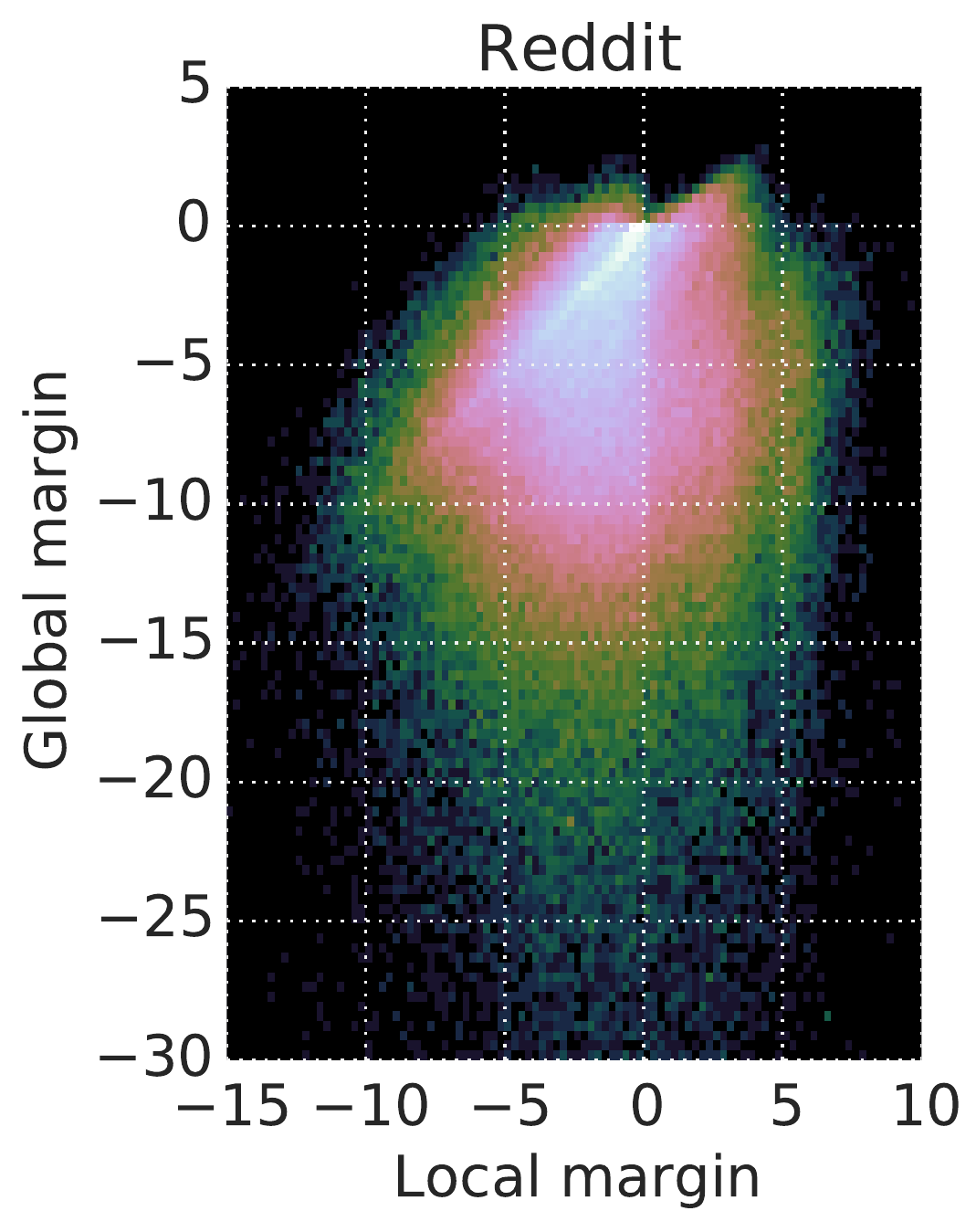}
  \includegraphics[width=0.23\textwidth]{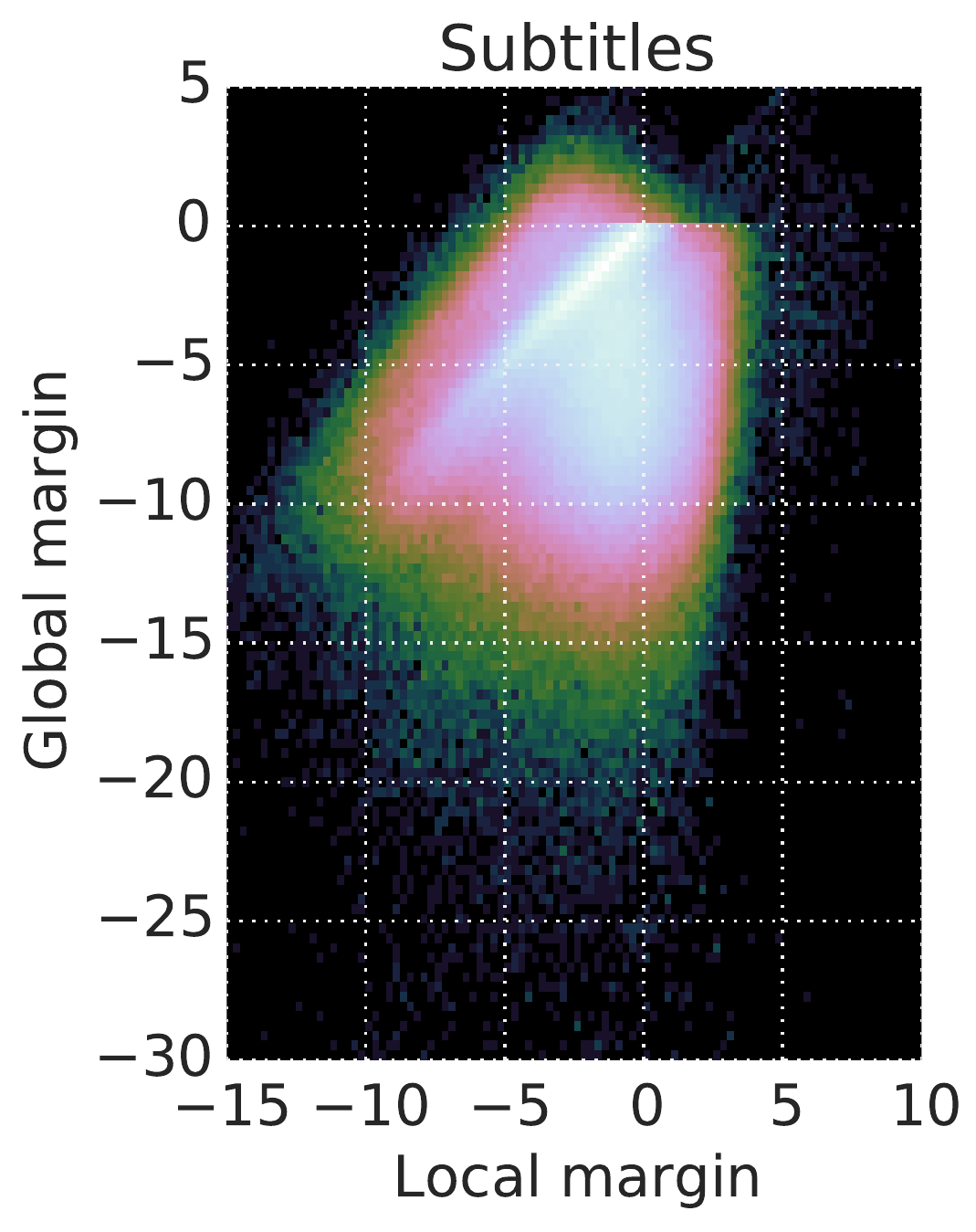}
  \caption{Local margin versus global margin for incorrectly predicted sequences. The color luminosity is proportional to frequency.}
  \label{fig:margin-scatter}
\end{figure}

\paragraph*{Margin discrepancy is higher for longer sequences}
In Figure~\ref{fig:margin} we show that this discrepancy is
significantly more pronounced for longer sequences.  In the figure we
show the fraction of wrongly predicted sequences with a positive local
margin.  We find that as sequence length increases, we have more cases
where the local margin is positive yet the global margin is negative.
For example, for the Reddit dataset half of the wrongly predicted
sequences have a positive local margin indicating that the training
loss was low for these sequences even though they were not adequately
separated.  
\begin{figure}[ht]
  \includegraphics[width=0.23\textwidth]{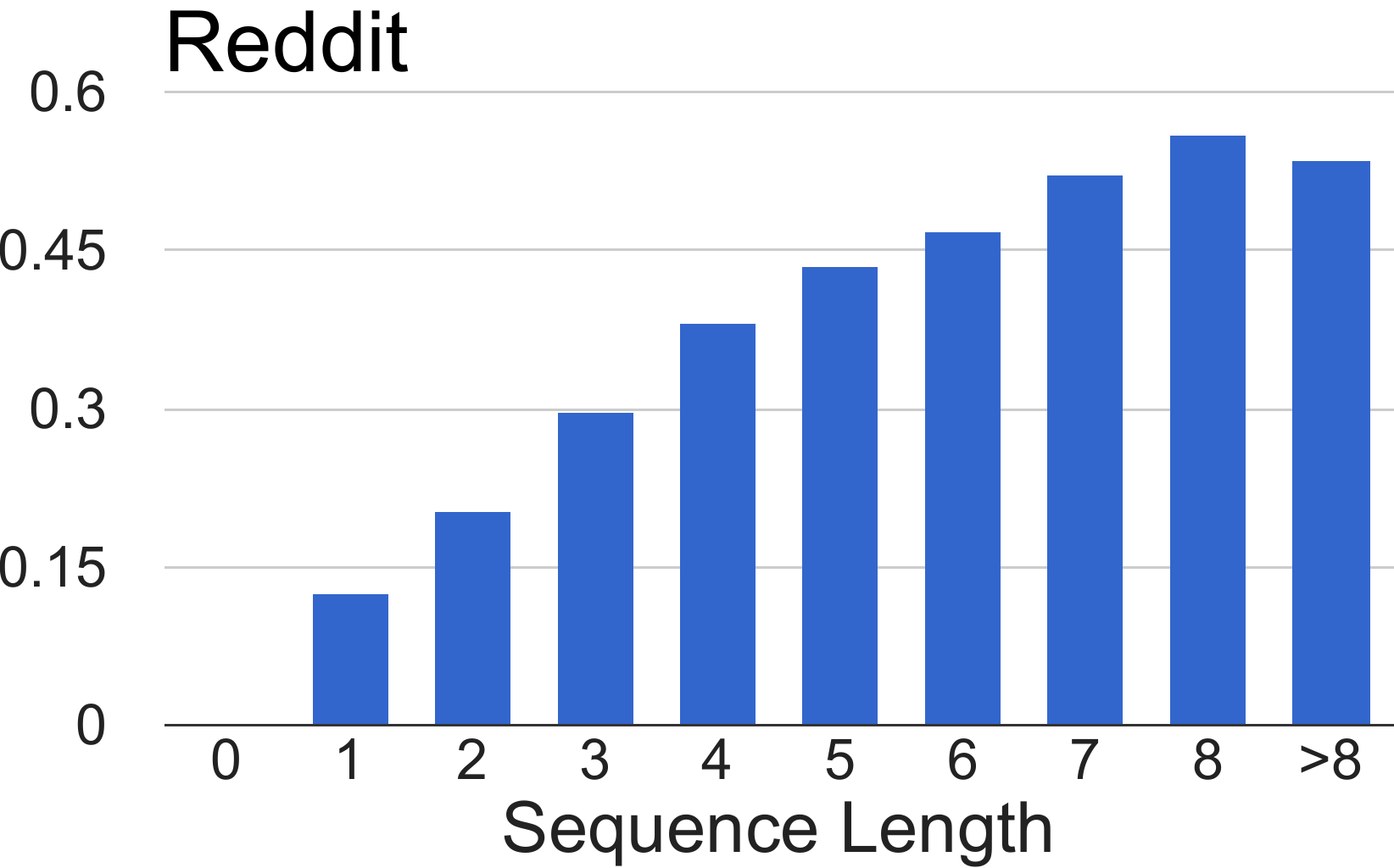}
  \includegraphics[width=0.23\textwidth]{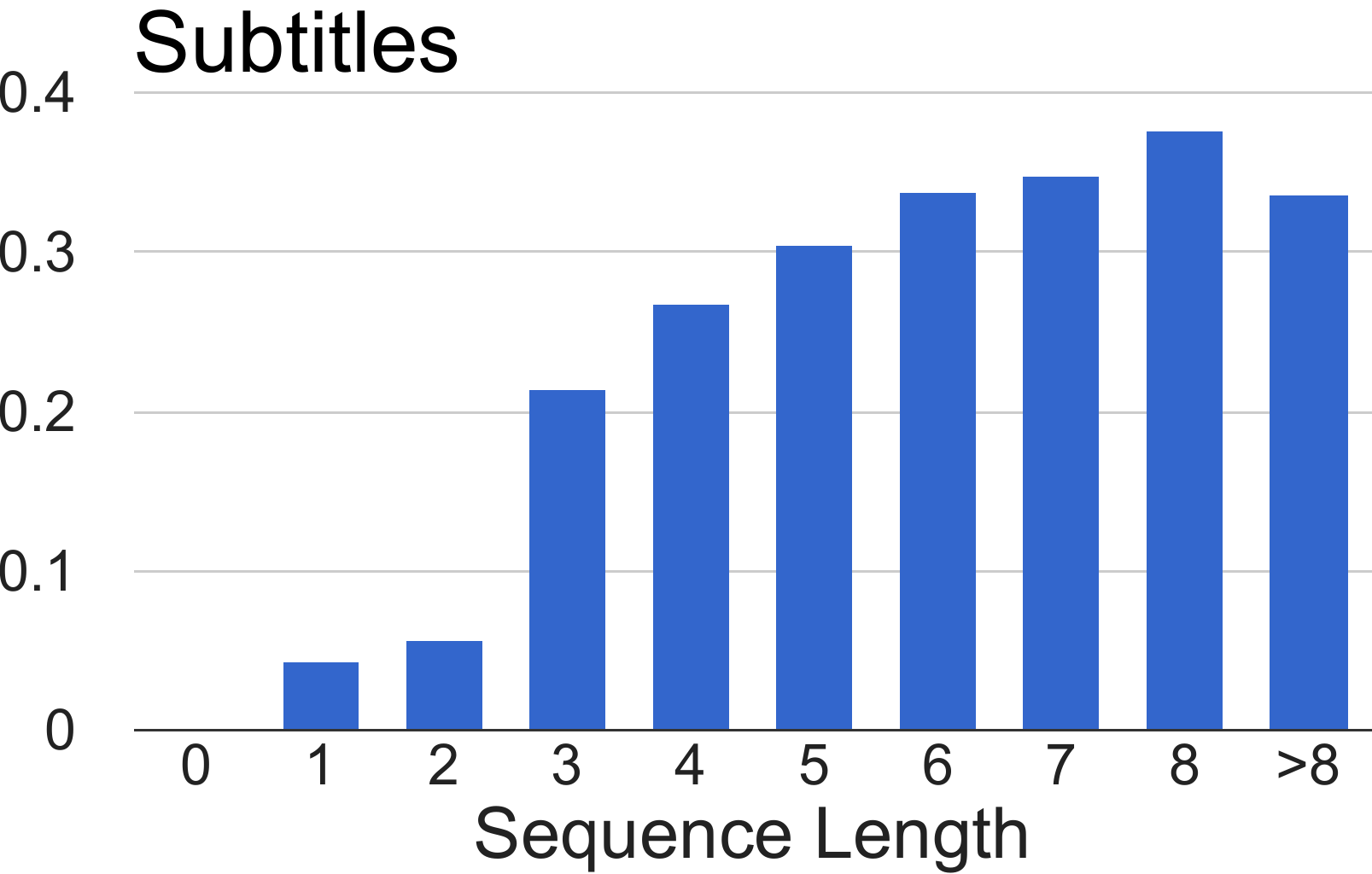}
  \caption{Fraction of incorrect predictions with positive local margin.}
  \label{fig:margin}
\end{figure}

\paragraph*{Increasing beam size drops accuracy for long sequences}
Next we show why this discrepancy leads to non-monotonic accuracies
with increasing beam-size.  As beam size increases, the predicted
sequence has higher probability and the accuracy is expected to
increase if the trained probabilities are well-calibrated.  In
Figure~\ref{fig:beam} we plot the number of correct predictions (on a
log scale) against the length of the correct sequence for beam sizes
of 1, 5, 10, and 15.  For small sequence lengths, we indeed observe
that increasing the beam size produces more accurate results.  For
longer sequences (length $> 4$) we observe a drop in accuracy with
increasing the beam width beyond 1 for Reddit and beyond 5 for
Subtitles.

\begin{figure}[ht]
  \includegraphics[width=0.4\textwidth]{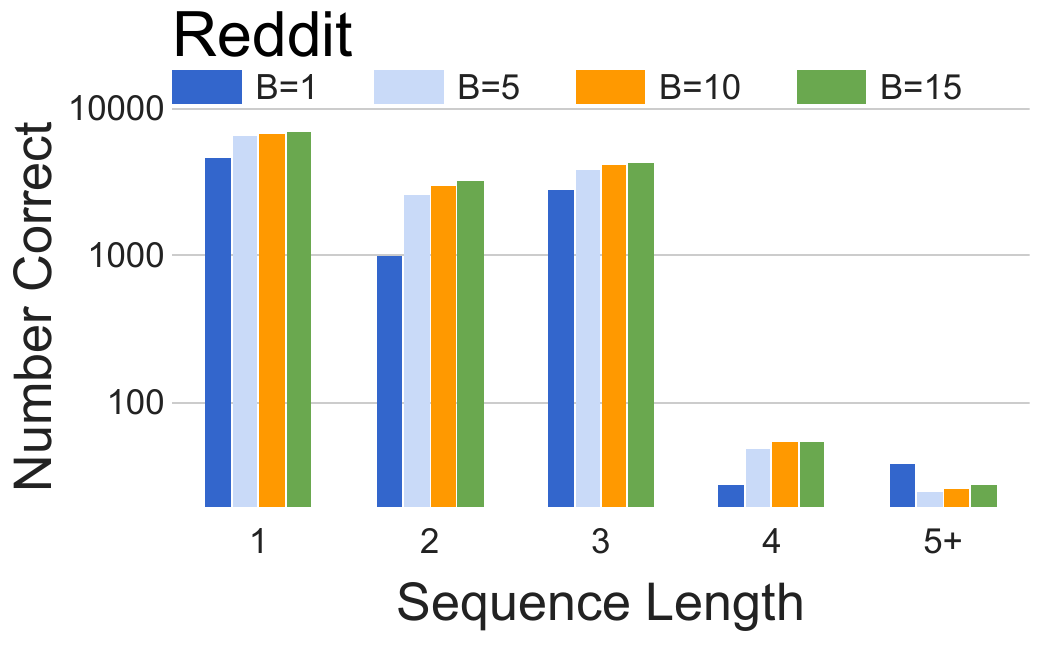}
  \includegraphics[width=0.4\textwidth]{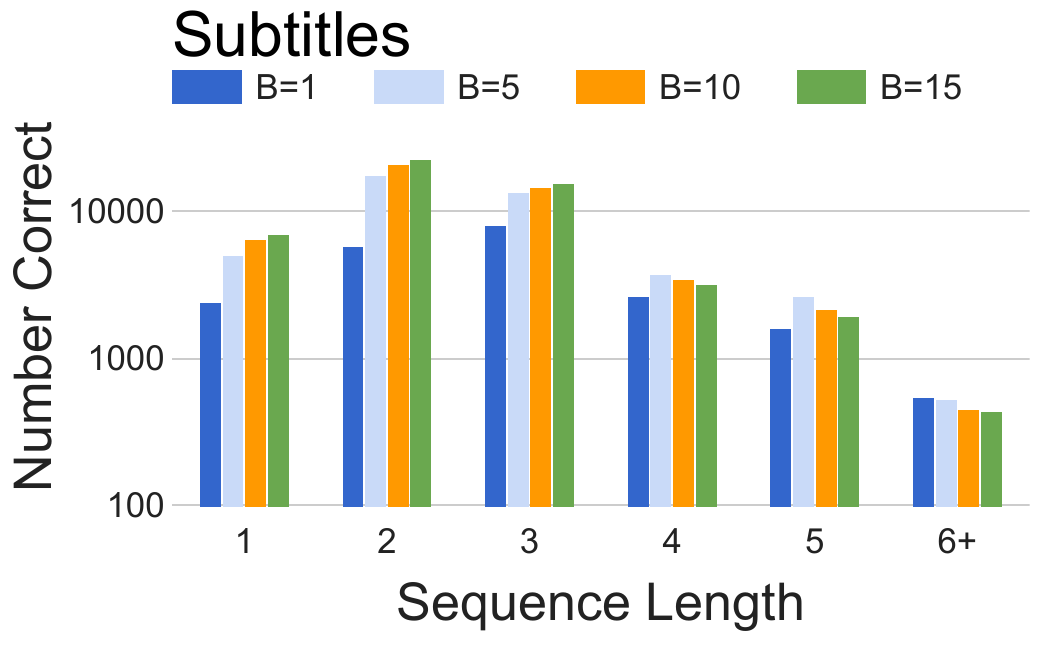}
    \caption{Effect of beam width on the number of correct predictions broken down by sequence length.}
  \label{fig:beam}
\end{figure}

\begin{figure*}[ht]
\begin{center}
  \includegraphics[width=0.32\textwidth]{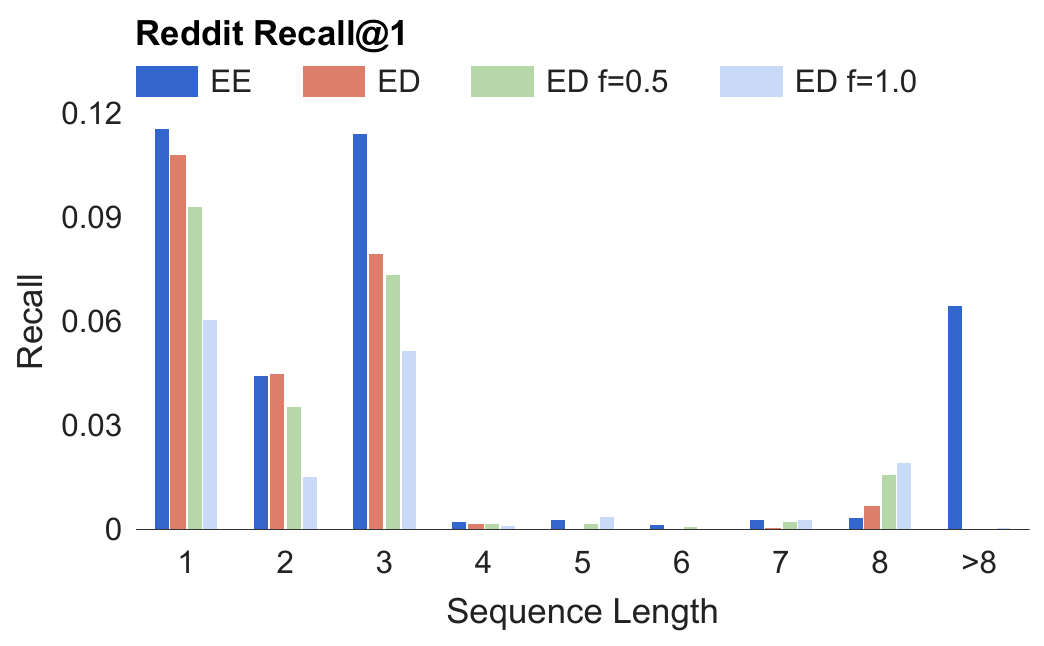}
  \includegraphics[width=0.32\textwidth]{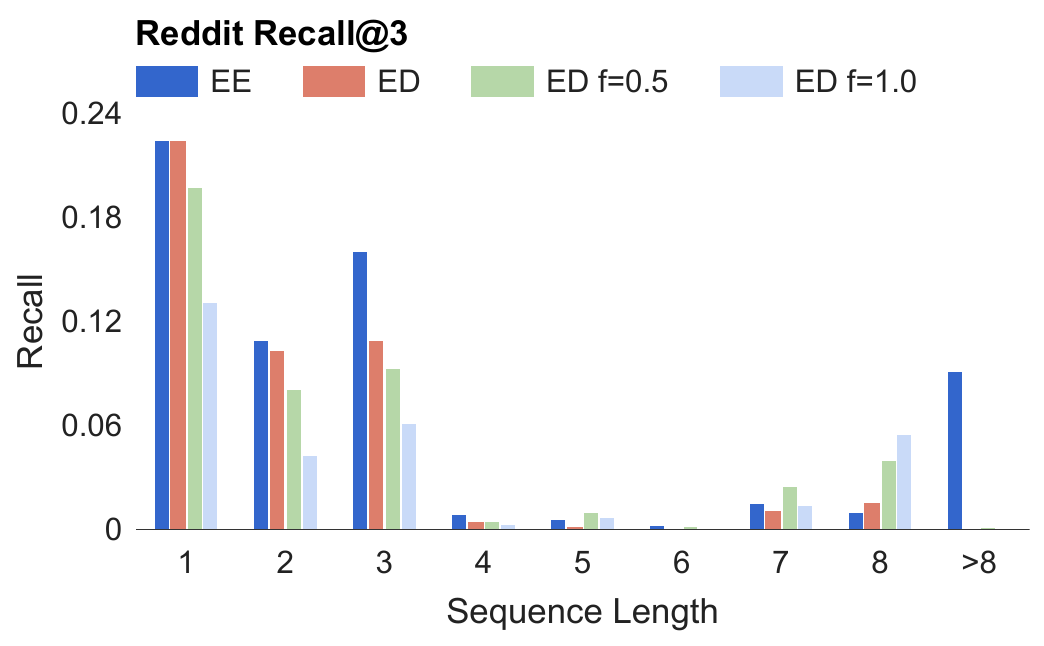}
  \includegraphics[width=0.32\textwidth]{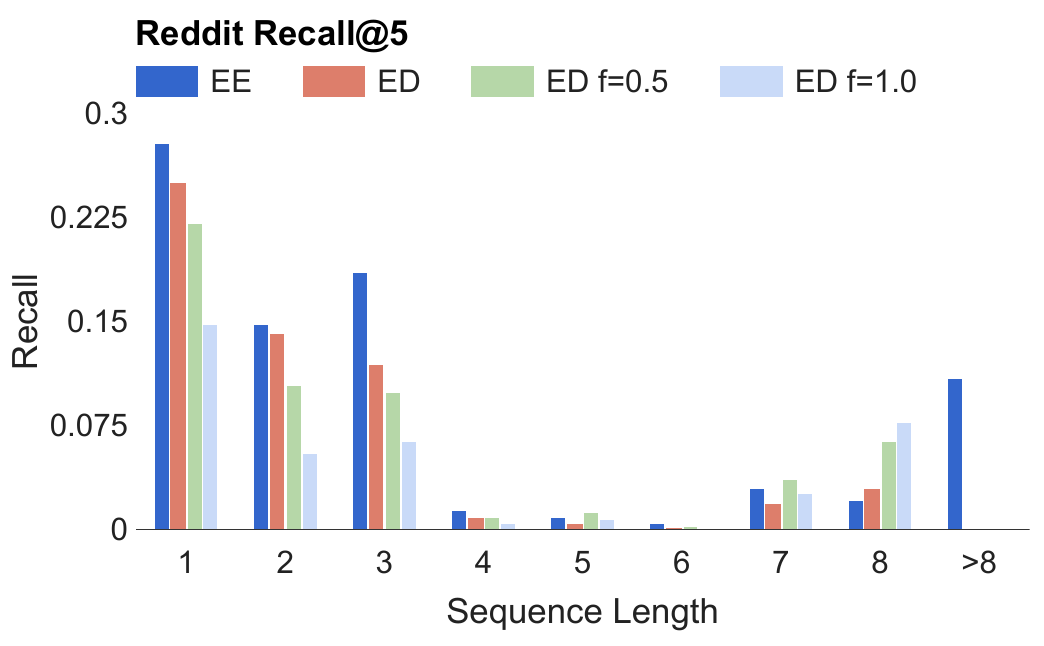}
  
  \includegraphics[width=0.32\textwidth]{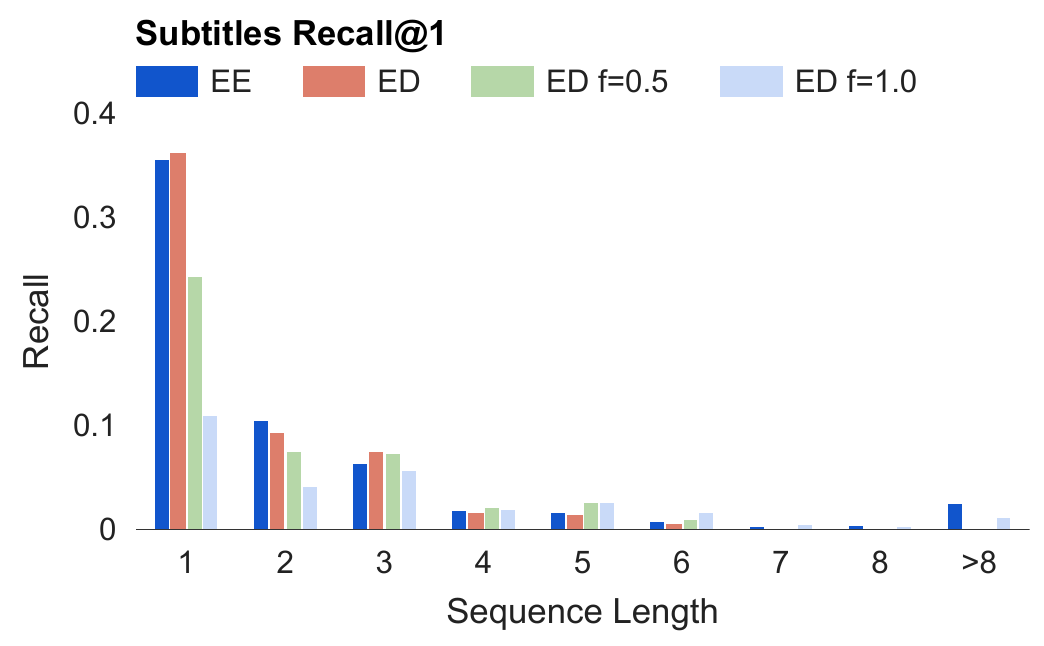}
  \includegraphics[width=0.32\textwidth]{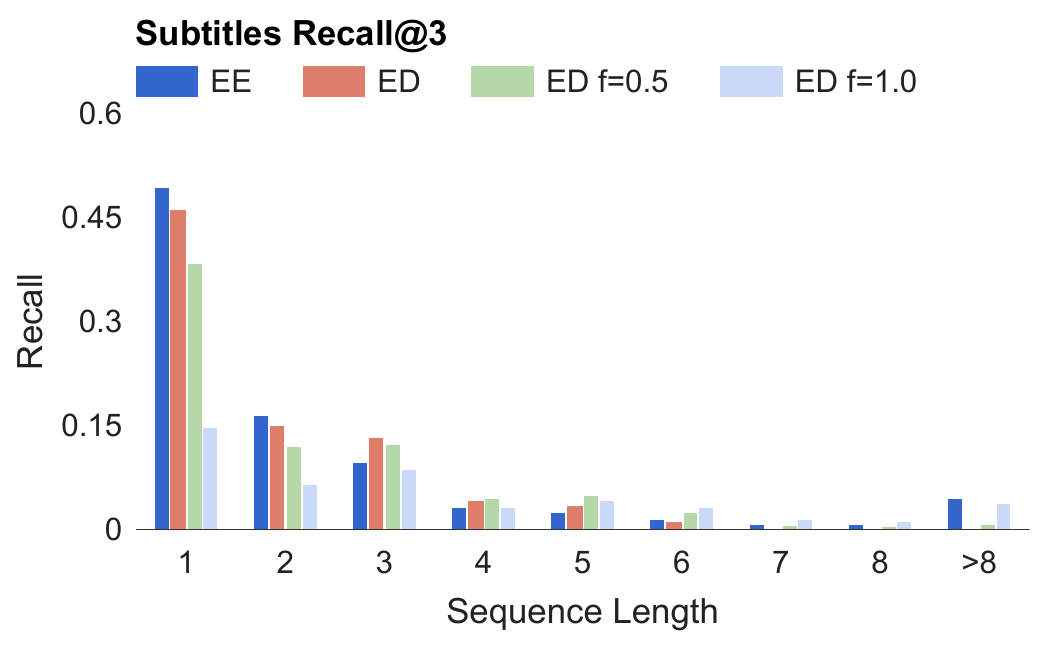}
  \includegraphics[width=0.32\textwidth]{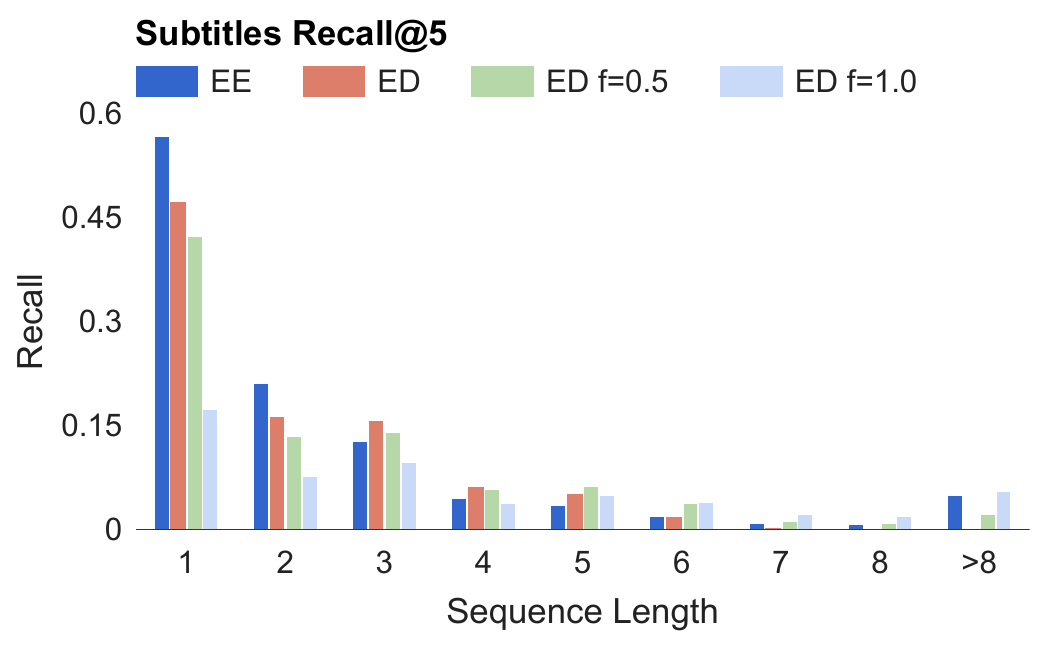}
  \caption{Comparing recall@1, 3, 5 for increasing length of correct sequence.}
  \label{fig:recall}
\end{center}
\end{figure*}

\paragraph*{Globally conditioned models are more accurate than ED models}
We next compare the ED model with our globally conditioned
encoder-encoder (EE) model.  In Figure~\ref{fig:recall} we show the
recall@K values for K=1, 3 and 5 for the two datasets for increasing
length of correct sequence.  We find the EE model is largely better
that the ED model.  The most interesting difference is that for
sequences of length greater than 8, the ED model has a recall@5 of
zero for both datasets.  In contrast, the EE model manages to achieve
significant recall even at large sequence lengths.

\paragraph*{Length normalization of ED models}
A common modification to the ED decoding procedure used to promote longer
message is normalization of the prediction log-probability by its length
raised to some power $f$ \cite{ChoMBB14,Graves13}. We experimented with two
settings, $f=0.5$ and $1.0$. Our experiments show that while this indeed
promotes longer sequences, it does so at the expense of reducing the accuracy
on the shorter sequences.

\section{Related Work}

In this paper we showed that encoder-decoder models suffer from length
bias and proposed a fix using global conditioning.
Global conditioning has been proposed for other RNN-based sequence
prediction tasks in \cite{YaoPZYLG14} and \cite{PezeshkiFBCB15}.  The
RNN models that these work attempt to fix capture only a weak form of
dependency among variables, for example they assume $\vx$ is seen
incrementally and only adjacent labels in $\vy$ are directly
dependent.  As proved in \shortcite{PezeshkiFBCB15} these models are
subject to label bias since they cannot represent a distribution that
a globally conditioned model can.  Thus, their fix for global
dependency is using a CRFs.  Such global conditioning will compromise
a ED model which does not assume {\em any conditional independence}
among variables.  The label-bias proof of \shortcite{PezeshkiFBCB15}
is not applicable to ED models because the proof rests on the entire
input not being visible during output. Earlier illustrations of label
bias of MeMMs in ~\cite{Bottou1991,Lafferty2001} also require local
observations.
% One conjecture is that since ED models locally normalize the
% probability around each token they are subject to label bias as in
% locally conditioned MeMM models.  However, all illustrations of
% label-bias use local observations for
% transitions~\cite{Bottou1991,Lafferty2001,PezeshkiFBCB15}.
In
contrast, the ED model transitions on the entire input and chain rule
is an exact factorization of the distribution. Indeed one of the
suggestions in \cite{Bottou1991} to surmount label-bias is to use a
fully connected network, which the ED model already does.

Our encoder-encoder network is reminiscent of the dual encoder network in
\cite{Pineau2015}, also used for conversational response generation. A crucial
difference is our use of importance sampling to correctly estimate the
probability of a large set of candidate responses, which allows us to use the
model as a standalone response generation system. Other  differences
include our model using separate sets of parameters for the two encoders, to
reflect the assymetry of the prediction task. Lastly, we found it crucial for
the model's quality to use multiple appropriately weighed negative examples for
every positive example during training.

\cite{marc2016sequence} also highlights limitations of the ED model
and proposes to mix the ED loss with a sequence-level loss in a
reinforcement learning framework under a carefully tuned schedule.
Our method for global conditioning can capture sequence-level losses
like BLEU score more easily, but may also benefit from a similar mixed loss function.

\section{Conclusion}
We have shown that encoder-decoder models in the regime of finite data
and parameters suffer from a length-bias problem. We have proved that
this arises due to the locally normalized models insufficiently
separating correct sequences from incorrect ones, and have verified
this empirically.  We explained why this leads to the curious
phenomenon of decreasing accuracy with increasing beam size for long
sequences.  Our proposed encoder-encoder architecture side steps this
issue by operating in sequence probability space directly, yielding
improved accuracy for longer sequences.

One weakness of our proposed architecture is that it cannot generate
responses directly. An interesting future work is to explore if the ED
model can be used to generate a candidate set of responses which are
then re-ranked by our globally conditioned model. Another future area
is to see if the techniques for making Bayesian networks
discriminative can fix the length bias of encoder decoder
networks~\cite{PeharzTP13,Dana2012}.

{\small
\bibliography{main}
} 
\bibliographystyle{emnlp2016}

\end{document}